\documentclass[twoside]{article}

\usepackage[accepted]{aistats2022}

\usepackage[OT1]{fontenc} 

\usepackage{graphicx} %
\usepackage{xcolor}
\usepackage{wrapfig}
\usepackage{float}
\usepackage{subcaption}
\usepackage{booktabs}
\usepackage[font=small]{caption} %

\usepackage[compact]{titlesec} 
\titlespacing*{\section}{14pt}{7pt}{4pt}
\titlespacing*{\subsection}{6pt}{3pt}{1pt}
\usepackage{amsmath,amssymb,amsthm}
\usepackage{color,mathtools}

\usepackage{multirow}
\usepackage{enumitem}

\usepackage{algorithm}
\usepackage{algpseudocode}

\usepackage{bm}
\newcommand{\bxi}{\boldsymbol{\xi}}

\usepackage[hyperindex,breaklinks]{hyperref}

\usepackage{amssymb,amsmath,amsthm,dsfont}

\providecommand{\lin}[1]{\ensuremath{\left\langle #1 \right\rangle}}

  \providecommand{\R}{\mathbb{R}} 
  
  \providecommand{\E}{{\mathbb E}}
  
  \renewcommand{\epsilon}{\varepsilon}

  \providecommand{\1}{\mathbf{1}}

  \renewcommand{\gg}{\mathbf{g}}

  \providecommand{\pp}{\mathbf{p}}

  \providecommand{\xx}{\mathbf{x}}
  \providecommand{\yy}{\mathbf{y}}

  \providecommand{\cM}{\mathcal{M}}
  
  \providecommand{\cO}{\mathcal{O}}
  \providecommand{\cP}{\mathcal{P}}

\RequirePackage[colorinlistoftodos,bordercolor=orange,backgroundcolor=orange!20,linecolor=orange,textsize=scriptsize]{todonotes}
\providecommand{\comment}[2]{\todo[inline,caption={}]{\textbf{#1: }#2}}%
\providecommand{\inlinecomment}[3]{%
  {\color{#1}#2: #3}}%
\newcommand\commenter[2]%
{%
  \expandafter\newcommand\csname i#1\endcsname[1]{\inlinecomment{#2}{#1}{##1}}
  \expandafter\newcommand\csname #1\endcsname[1]{\comment{#1}{##1}}
}

\newtheorem{lemma}{Lemma}

\newtheorem{remark}{Remark}
\newtheorem{assumption}{Assumption}
\newtheorem{theorem}[lemma]{Theorem}

\DeclarePairedDelimiterX{\inp}[2]{\langle}{\rangle}{#1, #2}
\DeclarePairedDelimiterX{\abs}[1]{\lvert}{\rvert}{#1}
\DeclarePairedDelimiterX{\norm}[1]{\lVert}{\rVert}{#1}
\DeclarePairedDelimiterX{\cbr}[1]{\{}{\}}{#1} 
\DeclarePairedDelimiterX{\rbr}[1]{(}{)}{#1} 
\DeclarePairedDelimiterX{\sbr}[1]{[}{]}{#1} 

\definecolor{mydarkblue}{rgb}{0,0.08,0.45}
\hypersetup{ %
    colorlinks=true,
    linkcolor=mydarkblue,
    citecolor=mydarkblue,
    filecolor=mydarkblue,
    }

\providecommand{\Msuper}{\cM_{\text{super}}}
\providecommand{\Mcore}{\cM_{\text{core}}}

\usepackage[sort,round]{natbib}

\usepackage[symbol]{footmisc}
\let\theoldfootnote\thefootnote
\renewcommand{\thefootnote}{\fnsymbol{footnote}}

\renewcommand{\cite}{\citep}

\begin{document}

\twocolumn[

\aistatstitle{Masked Training of Neural Networks with Partial Gradients} 

\aistatsauthor{Amirkeivan Mohtashami
\And
Martin Jaggi
\And
Sebastian U. Stich}

\aistatsaddress{ EPFL \And EPFL \And CISPA\footnotemark[2] } ]

\begin{abstract}

	State-of-the-art training algorithms for deep learning models are based on stochastic gradient descent (SGD).
    Recently, many variations have been explored:\ perturbing parameters for better accuracy (such as in Extragradient), limiting SGD updates to a subset of parameters for increased efficiency (such as meProp) or a combination of both (such as Dropout). However, the convergence of these methods is often not studied in theory. \\
	We propose a unified theoretical framework to study 
	such SGD  variants---encompassing the aforementioned algorithms and additionally a broad variety of methods used for communication efficient training or model compression. 
	Our insights can be used as a guide to improve the efficiency of such methods and facilitate generalization to new applications.
	As an example, we tackle the task of jointly training networks, a version of which (limited to sub-networks)
	 is used to create Slimmable Networks. By training a low-rank Transformer jointly with a standard one we obtain superior performance than when it is trained separately. 
\end{abstract}

\footnotetext[2]{CISPA Helmholz Center for Information Security. Most research carried out while at EPFL.}
\renewcommand{\thefootnote}{\theoldfootnote}

\section{Introduction}
Deep learning models are being used successfully in an increasing variety of scenarios. Consequently, new variants of SGD---the traditional training method for neural networks---have been designed, encouraged by various goals such as improving training efficiency~\cite{sun2017meprop}, improving final model's accuracy~\cite{korpelevich1976extragradient}, or imposing additional properties on the trained model \citep[e.g.\ being low rank][]{wang2021pufferfish}.
In distributed training, compressed versions of gradients are used in order to reduce communication costs~\cite{alistarh2017qsgd}. Dropout \cite{srivastava2014dropout}, a method that computes the gradient after deactivating a random subset of neurons in the network, is widely used and is known to yield models with better performance.
The Slimmable Nets approach \cite{yu2018slimmable, yu_universally_2019} allows adjusting the \emph{width} of a neural network dynamically during training and inference. While extensive literature exists on theoretical convergence bounds of SGD~\cite{Bottou2018:book,stich2019unified,Stich2021:critical},  variants of SGD, such as those we discussed earlier, often lack a theoretical analysis even though they perform well in practice. 
However, theoretical understanding of a method can significantly ease and guide its adaptation to different tasks as well as its combination with other methods.

\begin{figure*}[t]
	\centering
\includegraphics[width=.75\linewidth]{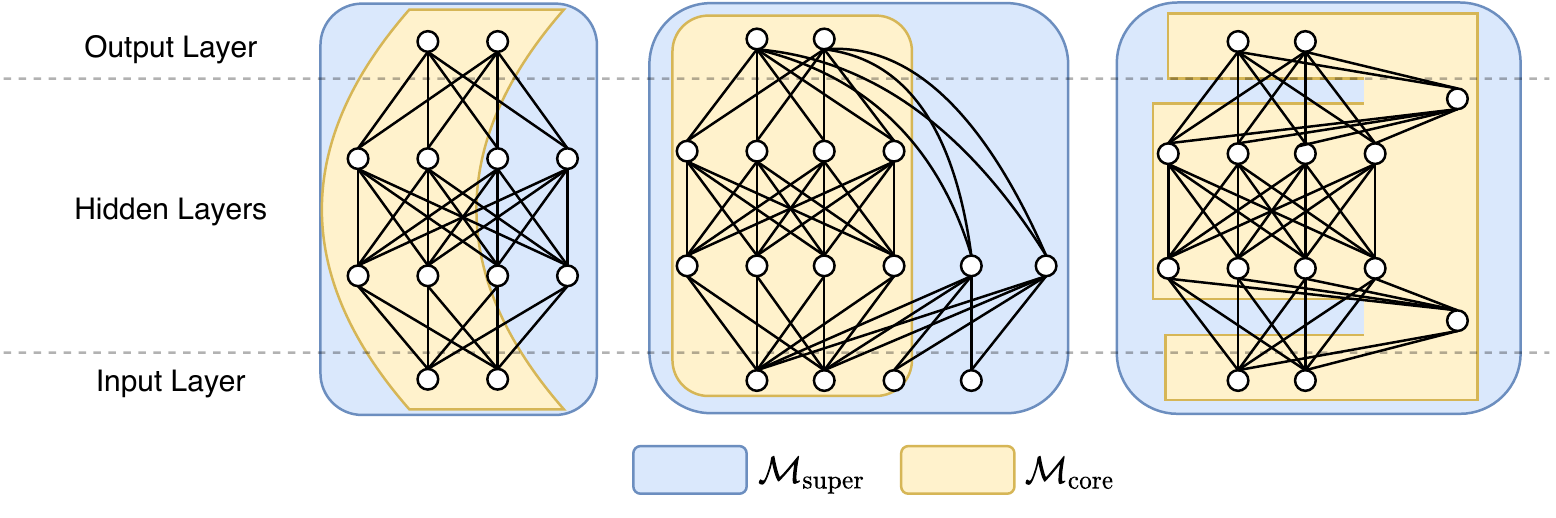}
\caption{Illustrations of relevant compositions of a core network $\Mcore$ ({\color{orange}yellow}) with a larger network $\Msuper$~({\color{blue}blue}).
 The left figure corresponds to the \emph{slimmable} case of $\Mcore$ being a small-width sub-network of $\Msuper$. The middle example additively combines the output of two networks of different shape, matching \emph{wide-and-deep networks} \cite{cheng2016wide}. The right architecture shows the a more complex scheme where several fully-connected layers are replaced by their low-rank variant in $\Mcore$, 
  which are combined in $\Msuper$ by adding the outputs layer-wise, while sharing the weights for the remaining layers.}\vspace{-2mm}
\label{fig:super-network-illustration}
\end{figure*}

In this paper, we propose a unified theoretical framework to study variants of SGD that rely on a combination of parameter perturbation and gradient masking. In particular, we theoretically analyze convergence of these algorithms under a generic template called partial SGD and identify a set of criteria that when satisfied yield a theoretical convergence guarantee.

Our result applies to a wide range of existing methods such as the Extragradient method \cite{korpelevich1976extragradient}, Dropout \cite{srivastava2014dropout}, extensions of Slimmable Nets \cite{yu2018slimmable}, partial backpropagation methods such as meProp~\cite{sun2017meprop} and parameter freezing methods such as layer-wise training~\cite{bengio2007greedy}, wide-and-deep networks \cite{cheng2016wide}, weight pruning \cite{lin2019dynamic}, as well as distributed independent subnet training~\cite{yuan2019distributed}, among others. We evaluate some of these methods based on our criteria and obtain rigorous convergence rates as well as insights on how these algorithms might be further improved.

More importantly, our results can also be used as a guideline to create new training methods (e.g.\ to impose special properties on the final model) since design choices can guarantee satisfaction of said criteria. 
As an example, we propose a training scheme to address the fundamental task of simultaneous training of two networks---a core network together with a larger architecture containing it.  This includes the case of joint training of arbitrary combinations formed by two different neural networks, in which case the larger network is the union of the two. Figure~\ref{fig:super-network-illustration} illustrates several example cases. 
Interestingly our designed method 
 is an unrolled version of slimmable training. However, whereas slimmable training was limited to training sub-networks with reduced width we apply our method to more generalized scenarios. Joint training can have various motivations including but not limited to allowing weight sharing across different networks, improving training speed, or improving final accuracy.
To demonstrate the practical effectiveness of our method, we train a low-rank core network for a Transformer architecture \cite{vaswani2017attention} together with the larger original architecture.
As a result of joint training, the trained low-rank model has superior stand-alone performance than when it is trained alone. The gain is particularly pronounced for significant model size reductions (32x rank reduction) where computational cost savings are highest (17.95 vs.\ increase to 24.03 BLEU score). 
Furthermore, the larger architecture can be converted to a full-rank standard Transformer matching the performance of the highly-optimized original Transformer model.

\noindent\textbf{Contributions} Our main contributions can be summarized as follows: 
\begin{itemize}[nosep,leftmargin=12pt,itemsep=1pt]
	\item We provide a rigorous theoretical analysis of partial SGD---a generic template encompassing  training schemes that involve a combination of parameter perturbation %
	and gradient masking.
	\item We present and discuss a host of existing methods that fit our theoretical framework.
Our detailed discussion illustrates how our insights can guide future design decisions or lead to improvements for existing methods %
 (showcased by designing a training method that jointly trains a full model and a reduced capacity model of much smaller size).
	 \item We generalize slimmable training to go beyond just reducing width but allowing arbitrary sub-networks, 
	  and show the effectiveness of joint training in various practical scenarios, including training Transformer models together with low-rank variants on NLP tasks.
\end{itemize}

\section{Related Work}

The theoretical framework developed in this work covers a broad set of training methods used in a wide variety of settings. In distributed training, a worker can be responsible for optimizing a subnetwork while the local updates are expected to contribute to optimization of the full network~\cite{yuan2019distributed}. Even in the data-parallel setting where each worker optimizes the full model but on its own local dataset, a compressed (masked) version of gradient is often used for communication 	efficiency~\cite{dryden2016communication,aji2017sparse}. A similar technique is applied in  \cite{sun2017meprop} to improve speed of single-node training. Moreover, certain parameter perturbations seem to be effective in improving a model's ability to generalize, such as in extragradient methods \cite{korpelevich1976extragradient}. A combination of masking and parameter perturbation can be seen in Dropout~\cite{srivastava2014dropout}, a widely used technique which also fits our framework. 

In this work, we show convergence bounds for algorithms that conform to our partial SGD template in smooth non-convex settings which is widely used in the SGD convergence literature \cite{Bottou2018:book,Arjevani2018:delayed}. The convergence of such algorithms is also investigated under NTK assumptions~\cite{jacot2021neural} for shallow neural networks \cite{liao2021convergence} but only for random masking applied for both gradient computation and masking. In contrast, we allow arbitrary parameter perturbations for gradient computation, as well as arbitrary gradient masking which might not be random nor necessarily corresponding to said perturbation. 
Other works investigate convergence of instances of our generic template, such as Dropout \cite{senen2020asymptotic,senen2020almost,jacot2021neural}, or gradient compression \cite{stich2020error,stich2018sparsified,alistarh2018convergence}.

Our framework can also be used to guide the design of new methods. We illustrate this by designing a method for joint training. The resulting method resembles an unrolled version of Slimmable training \cite{yu2018slimmable, yu_universally_2019}. Slimmable networks are trained such that the width can be chosen dynamically based on available resources during inference. Dynamic sparsity neural networks \cite{wu2021dynamic} extend this paradigm to subnetworks with different levels of sparsity, allowing removal of individual weights instead of neurons. In separate work simultaneous to ours,  \citet{peste2021ac} apply a similar method, also used in \cite{jin2016training}, %
to train sparse models. They theoretically prove convergence in the case when sparse minimizers exist, such as in overparameterized networks. While we only perturb the parameters for computing the gradient vector and apply the updates on the original values, their method applies the perturbation permanently; zeroing the masked parameters whereas we keep them unchanged. Moreover, our setting is more general since we allow perturbations other than masking.

An example of new applications considered in our work is training a network jointly with its low-rank variant. Previous work in this area includes adding regularization to the objective loss to encourage low-rank learning \cite{wen_coordinating_2017}. Another approach is used in \cite{wang2021pufferfish} to train the low-rank variant by initially using the full network to jump-start training. In contrast, our method continuously trains both the full network and the low-rank variant without any changes to the loss. 

It is also possible to obtain a good small network from a larger network using methods such as knowledge distillation \cite{hinton2015distilling,linS2020distillFL} or model pruning \cite{han2015learning,zhu2017prune,lin2019dynamic}. Distillation improves the training of a student network (here the small network) by using the outputs of a teacher network (here the larger network). However, this requires training the small and large networks separately, which is more costly than our method. On the other hand, model pruning methods allow extracting a small network with good performance from a large network. While these methods can be fast and sometimes remove the need for training the small network, it is usually not possible to impose a structure on the small network.  In contrast, our method allows a specifying a precise pre-defined architecture as the small network.  A different line of work, looks for lottery ticket sub-networks in a normally trained network that can be used separately or trained from scratch to obtain the same performance \cite{frankle2018lottery,frankle2019stabilizing}. However, these sub-networks often lack a hardware-compatible structure that can hinder obtaining a performance boost. In contrast, joint training can intuitively be seen as shaping the lottery ticket. 

Previous work successfully use joint training for neural architecture search where multiple networks are trained at once and the right architecture is determined by comparing their performance~\cite{cai2019once,yu2019autoslim,yu2019evaluating,yu2020bignas}. As an example of fusing two network together, Wide-and-deep networks \cite{cheng2016wide} consist of a wide and another deep but narrow network with the output being the sum of the two networks. The wide part is a linear model that e.g.\ uses carefully selected features as well as their cross-product as input, at a very low inference latency cost. On the other hand, the deep part works on the raw features which allows the wide part to memorize more complex patterns of interactions between features %
while the deep part can help generalization~\cite{cheng2016wide}.  %
\citet{guo2017deepfm} extend these networks by removing the need for manual feature engineering. %
\looseness=-1

Finally, in another line of research, the implicit bias caused by applying these methods is investigated \cite{hooker2020characterising}. We leave further research of this aspect of our method to future work. \looseness=-1

\looseness=-1

\begin{figure}
	\begin{algorithm}[H]
		\caption{\textsc{partial SGD}%
		} 
		\label{alg:partialsgd}
		\resizebox{\linewidth}{!}{
			\begin{minipage}{1.1\linewidth}
				\begin{algorithmic}[1]
					\For{$t=0,\dots,T$}
					\State Choose a mask $\pp_t$, and
					\State a perturbation $\delta\xx_t$ for the current step $t$
					\State $\tilde{\xx}_t \gets \xx_t + \delta\xx_t$ 
					\Comment perturb parameters
					\State Take a stochastic gradient $\gg_t$ of $\nabla f(\tilde{\xx}_t)$
					\State  $\xx_{t + 1} \gets \xx_t - \gamma_t \pp_t \odot \gg_t $ 
					\EndFor
				\end{algorithmic}
		\end{minipage}}%
	\end{algorithm}%
	\vspace{-2em}
\end{figure}

\section{Convergence Analysis}
\label{sec:theory}
We consider the optimization of the full network parameters $\xx$. In formal terms, we consider finding the minima of the empirical loss $f \colon  \R^{d} \to \R$ :\vspace{-1mm}
\begin{equation} \textstyle
	f^\star := \min_{\xx \in  \R^{d}} f(\xx) \,.
\end{equation}

\subsection{Partial SGD}

We start by introducing partial SGD (Algorithm~\ref{alg:partialsgd}), a generic template designed to fit various training algorithms that apply a combination of gradient masking, and parameter perturbation for gradient computation. In particular, we make the following changes compared to vanilla SGD:

\begin{enumerate}[leftmargin=16pt,nosep,itemsep=1pt]
	\item We allow a time-dependent binary mask $\pp_t$ in each step to be chosen arbitrarily from a set of binary masks $\cP \subseteq \{0, 1\}^d$. Selecting the all-one vector in every step recovers vanilla SGD. Another example is a method that picks a random mask each time.
	\item We allow arbitrary perturbation $\delta \xx_t$ of the current parameters for calculating the loss and its gradient. Naturally, setting $\delta \xx_t = 0$ recovers vanilla SGD while setting $\delta \xx_t = -(\1 - \pp_t) \odot \xx_t$ would mean applying the gradient of the sub-network induced by the mask $\pp_t$. 
\end{enumerate}
Partial SGD can represent various training schemes.
As an example consider the case of applying dropout with probability~$\mu$ over the network. In this case $\pp_t$ is a random mask where each element (or each group of elements if the dropout is applied to neurons instead of individual weights) is one with probability~$\mu$. The perturbation $\delta \xx_t$ should be set equal to $-(\1 - \pp_t) \odot \xx_t$ (similar to alternating training scheme) to compensate for the masking of weights during forward propagation while $\pp_t$ limits the parameter update during backward propagation.

\subsection{Assumptions \& Definitions}\label{sec:assumptions}
We focus on the non-convex setting and assume the function $f$ to be $L$-smooth. 
\begin{assumption}[$L$-smoothness]
	\label{ass:lsmooth}
	The function $f \colon \R^d \to \R$ is differentiable and there exists a constant $L > 0$  such that:\vspace{-1mm}
	\begin{align}
		\norm{\nabla f(\xx)- \nabla f(\yy)} &\leq L \norm{\xx - \yy}\,, \qquad \forall \xx,\yy \in \R^{d} \,. \label{eq:lsmoothgrad}
	\end{align}
\end{assumption}

We assume that for every point in $\xx \in \R^d$ we can query a stochastic gradient $\gg(\xx)$ of $f(\xx)$, that is
\begin{align}
 \gg(\xx) := \nabla f(\xx) + \bxi(\xx) \,, \label{eq:goracle}
\end{align}
for a zero-mean stochastic noise vector $\bxi \colon \R^d \to R^d$.

We assume that the noise is bounded. Since we are interested in masking the gradient for partial training, we use a modified assumption over the set of allowed masks $\cP \subseteq \{0, 1\}^d$.
\begin{assumption}[$(M,\sigma^2)$-bounded noise over $\cP$]\label{ass:noise}
There exist parameters $M \geq 0, \sigma^2 \geq 0$ such that for any mask $\pp \in \cP$: \vspace{-2mm}
\begin{align*}
 \E{ \norm{\pp  \odot \bxi (\xx)}^2 }  &\leq M\norm{\pp \odot \nabla f(\xx)}^2 + \sigma^2\,, \qquad \forall \xx \in \R^{d}\,.
\end{align*}
\end{assumption}
\begin{remark}
When $\cP = \{\1\}$, the assumption becomes the same as the standard assumption used for analyzing SGD convergence in previous work \citep[e.g.][]{Bottou2018:book}. When the size of $\cP$ is larger, the assumption becomes stronger. For example, if $\cP = \{\1, \1_{\Mcore}\}$, where $\Mcore$ is a sub-network, the norm of noise over the sub-network is separately bounded by the norm of the gradient of the sub-network parameters. In its strongest form, i.e.\ when~$\cP$ contains all the possible masks, the assumption implies that the noise over each parameter is separately bounded by the norm of that parameter's gradient which is still a reasonable assumption.  \looseness=-1
\end{remark}

\subsection{Main Result}

We state the following main convergence theorem for non-convex objectives:

\begin{theorem}\label{thm:main}
	\label{thm:fullgrad}
Let Assumptions~\ref{ass:lsmooth}--\ref{ass:noise}, 
hold,
and let the stepsize $\gamma_t = \alpha_t \gamma_{base}$ in Algorithm~\ref{alg:partialsgd} with
$\gamma_{base}= \min\big\{\frac{1}{L(M+1)}, 
\frac{\epsilon}{L\sigma^2}
 \big\}$ for any $\epsilon > 0$ and 
 with $\alpha_t = \min\{1, \frac{\lin{\pp_t \odot \nabla f(\xx_t), \pp_t \odot \nabla f(\tilde{\xx}_t)}}{\norm{\pp_t \odot \nabla f(\tilde{\xx}_t)}^2}\}$, where $\pp_t$ denotes the binary mask and $\tilde{\xx}_t$ the perturbed parameters. Then,
\begin{itemize}[nosep,itemsep=1pt,leftmargin=12pt]
	\item
$\frac{1}{T}\sum_{t = 0}^{T - 1} \E_{\bxi} \alpha_t^2\norm{\pp_t \odot \nabla f(\tilde{\xx}_t)}^2 < \epsilon$, after at most the following number of iterations $T$:
\begin{align*}
\cO \left(\frac{\sigma^2}{\epsilon^2}  + \frac{(M + 1)}{\epsilon} \right) \cdot L F_0 \,.
\end{align*}
\item
Let $q := \max_{t \in [T], \bxi} \Big(q_t := 
\frac{\norm{\nabla f(\xx_t)}}{\norm{\pp_t \odot \nabla f(\xx_t)}} \cdot
\max\big\{\frac{\norm{\pp_t \odot \nabla f(\xx_t)}}{\norm{\pp_t \odot\nabla f(\tilde{\xx}_t)}}, \frac{\norm{\pp_t \odot \nabla f(\xx_t)}\norm{\pp_t \odot\nabla f(\tilde{\xx}_t)}}{\lin{\pp_t \odot \nabla f(\xx_t), \pp_t \odot \nabla f(\tilde{\xx}_t)}}\big\}\Big)$. Then  $\frac{1}{T}\sum_{t = 0}^{T - 1} \E_{\bxi}\norm{\nabla f(\xx_t)}^2 < \epsilon$ after at most the following number of iterations $T$:
\begin{align*}
	\cO \left(\frac{q^4\sigma^2}{\epsilon^2}  + \frac{q^2(M + 1)}{\epsilon} \right) \cdot L F_0 \,,
\end{align*}
\end{itemize}\vspace{-3mm}
where  $F_0:= f(\xx_0) - f^\star$ and we use $\bxi$ to denote the full sequence $\bxi(\xx_0), \bxi(\xx_1), \ldots, \bxi(\xx_{T - 1})$.
\end{theorem}

We highlight that the bounds proven by this theorem depend on the values of the stepsize scaling $\alpha_t$ and gradient alignment $q_t$. In particular, if the values of $\alpha_t$ are sufficiently large,  the first part of the theorem shows convergence of the partially masked core networks. Moreover, sufficiently small~$q$ ensures  convergence of the  full network. Note that these two values are inversely proportional ($q_t =\frac{\norm{\nabla f(\xx_t)}}{\alpha_t\norm{\pp_t \odot \nabla f(\tilde{\xx}_t)}}$) so that usually a large $\alpha_t$ results in a small $q_t$ and vice-versa. 

In reality, %
 the values $q_t$ and $\alpha_t$ also depend on the selection of masks which vary depending on training scheme and application.
The above theorem holds under the most general set of assumptions for widest applicability. More informative upper bounds on the value $q$ can be obtained in relevant applications. Note that Theorem~\ref{thm:main} naturally recovers the standard SGD convergence theorem~\cite{Bottou2018:book} as a special case, by setting $\pp_t = \1_d$. In this case $\alpha_t = q_t = 1$ and we recover the best known SGD bounds.

We also remark that while the theoretical result assumes that the learning rate is multiplied by $\alpha_t$, in practice it is not necessary to compute an exact value for $\alpha_t$ in each step. Instead, a lower bound $\alpha$ can be tuned as a hyper-parameter (which might also be different for different masks). Even simpler, in our experiments, we observed that no tuning is necessary, and using the same learning rate as for the standard training of the full network performs very well.

 We leave the proof of the theorem (leveraging techniques from the proof of BiasedSGD \cite{ajalloeian2020analysis}) to Appendix~\ref{app:theorem-proof}. In the following we propose a set of criteria that can be used to ensure the value $q$ for a training algorithm is small. As such they can be used both as a guide for developing new methods as well as proving convergence of existing methods.
 
 \newtheorem{criteria}{Criterion}

\begin{criteria}[Gradient norm preservation against masking]\label{crit:norm-masking} To ensure a small $q$ value, the binary masks applied on the gradient vector should not drastically reduce its norm, i.e.\ the ratio  $\frac{\norm{\nabla f(\xx_t)}}{\norm{\pp_t \odot \nabla f(\xx_t)}} \leq c_{\text{\normalfont norm}}$ should be bounded by a small constant $c_{\text{\normalfont norm}}$.  Note that this criterion is automatically satisfied with $c_{\text{\normalfont norm}}=1$ when no masking is done. 
	\end{criteria}

\begin{criteria}[Gradient norm preservation against perturbation]\label{crit:norm-perturbation} When applying perturbations before computing the gradient, the training algorithm should ensure that the perturbation does not lead to a drastic decrease in gradient norm, i.e., there should exist a constant $c_{\text{\normalfont sim}}$ such that $ \frac{\norm{\pp_t \odot \nabla f(\xx_t)}}{\norm{\pp_t \odot\nabla f(\tilde{\xx}_t)}} \leq c_{\text{\normalfont sim}} $. This  holds with $c_{\text{\normalfont sim}}=1$ when no perturbation is done.
\end{criteria}

\begin{criteria}[Gradient alignment]\label{crit:gradient-alignment} The gradient vectors computed with and without perturbation should be aligned. More formally, there should exist an upper bound $c_{\text{\normalfont align}}$ such that $ \frac{\norm{\pp_t \odot \nabla f(\xx_t)}\norm{\pp_t \odot\nabla f(\tilde{\xx}_t)}}{\lin{\pp_t \odot \nabla f(\xx_t), \pp_t \odot \nabla f(\tilde{\xx}_t)}} \leq c_{\text{\normalfont align}}$. This criterion also automatically holds with $c_{\text{\normalfont align}} =1$ when no perturbation is applied, i.e.\ when $\xx_t = \tilde \xx_t$.
\end{criteria}

We will now example various use-cases of these criteria.

\paragraph{meProp} meProp is a training scheme which only uses the top-$k$ elements of the gradient vector for optimization \cite{sun2017meprop}. Since no perturbation is applied, Criteria~\ref{crit:norm-perturbation} and \ref{crit:gradient-alignment} are satisfied. Choosing elements with the highest value in the gradient vector ensures Criterion~\ref{crit:norm-masking} as the upper bound $c_{\text{norm}} \leq \frac{d}{k}$ 
holds.  Note that if instead of proving the existing method we were looking for a method which only uses a portion of the gradient vector, this criterion would guide us to choose the elements with the highest value, guiding us toward the right method.  For completeness, we point out that, in practice,  meProp applies the top-$k$ approximation layer-wise during the backward pass. This speeds up the back-propagation process but leads to using an incorrect gradient in the lower layers, thus only being an approximation of our variant using the actual gradient.

\paragraph{Extragradient Method} In this method, the gradient is calculated after making a small step (with step size $\eta_s$) in the direction of the gradient, i.e.\ $\delta \xx_t = \eta_{s} \nabla f(\xx_t)$. Setting $\pp_t = \1$, recovers the Extragradient method from Algorithm~\ref{alg:partialsgd} which is a more generalized version that allows gradient masking as well as  parameter perturbations that are not in the direction of the gradient. Since no masking is done, criterion \ref{crit:norm-masking} is automatically satisfied. In order to show the other criteria are also satisfied, we introduce the following assumption and lemma.

\begin{assumption}[bounded perturbation]\label{ass:perturbation}
	For a $L$-smooth function, i.e.\ satisfying Assumption~\ref{ass:lsmooth}, the training algorithm makes perturbations such that      %
	\begin{align}
		\label{eq:perturbation}
		\max_{t \in [T]} \frac{\norm{\delta\xx_t}}{\max\big\{\norm{\pp_t \odot\nabla f(\xx_t)}, \norm{\pp_t \odot\nabla f(\tilde{\xx}_t)}\big\}} < \frac{1}{2L} \,.
	\end{align}
\end{assumption}

Note that by choosing a small perturbation coefficient, in particular $\eta_s \leq \frac{1}{2L}$,
 the assumption holds for the extragradient method since $\norm{\delta \xx_t} = \eta_s \norm{\nabla f(\xx_t)}$. The following lemma guarantees the other criteria.

\begin{lemma}
	\label{lemma:bounded-perturbation}
	If Assumptions~\ref{ass:lsmooth} and \ref{ass:perturbation} hold, $c_{\text{align}} \leq \sqrt{10}$ and $c_{\text{sim}} \leq \frac{\sqrt{10}}{2}$.
\end{lemma}

We postpone the proof of the above lemma to the Appendix~\ref{app:proof-bounded-perturbation-assumption}. Note that the above lemma guarantees $q \leq 10$ for the extragradient method yielding convergence bounds with the same complexity as previously known results \cite{xu2019extragradient}. Assumption~\ref{ass:perturbation} can also be useful in other settings, especially because it can be verified during training. Therefore, it is possible to design training algorithms which determine the perturbation and the mask based on whether this assumption holds. We now provide one possible example in the settings of model parallel training.

\paragraph{Model-Parallel Training}
One of the recently introduced methods for model-parallel distributed training of a network across several workers is independent subnet training \cite{yuan2019distributed}. In this method, neurons in each layer of the network are randomly partitioned into $k$ disjoint sets (equal to the number of workers), partitioning the whole network into $k$ disjoint parts. Each worker trains one part of the network by independently performing SGD steps on its personal core part for a limited number of steps. Next, the new weights are communicated, and the network is again randomly re-partitioned, each worker receiving a new part of the network. From the perspective of a single worker, the training scheme matches the use of a fixed dropout mask for several steps, which exactly fits our Algorithm~\ref{alg:partialsgd}. %
Moreover, because of the disjoint partitioning, the result is identical if the worker steps were interleaved instead of being run in parallel.
 More specifically, if $\pp_{i, k}$ corresponds with the mask for worker $i$ at step $k$, the sequence $\pp_t$ can be set to $\pp_{1, 1}, \pp_{2, 1}, \ldots, \pp_{k, 1}, \pp_{1, 2}, \ldots$. Hence, Theorem~\ref{thm:main} can be used to analyze the convergence of this scheme. It might be possible to evaluate the satisfaction of the above criteria using theoretical arguments or practical measurements. However, here we want to point toward another application of these criteria. In particular, we note that currently the algorithm requires tuning the number of local steps on each worker. Another approach is to verify Assumption~\ref{ass:perturbation} and change the partitioning whenever it is violated. Note that this is possible with minimal communication (norm of the weights) between workers. Even the communication might not be necessary since the norm of the weights does not significantly change in a single partitioning. We leave the practical investigation into this as a future work. %

\paragraph{Dropout} In the case of dropout, each element of the mask $\pp_t$  (or each group of elements corresponding to the weights connected to a specific neuron) is i.i.d.\ chosen from a Bernoulli distribution with mean~$\mu$.  In this case, the expected value of $\frac{\norm{\nabla f(\xx_t)}^2}{\norm{\pp_t \odot \nabla f(\xx_t)}^2}$ is $\frac{1}{\mu}$ which means Criterion~\ref{crit:norm-masking} is on average satisfied. We use a practical approach to ensure the other criteria holds. Note that it might also be possible to show the other criteria are also satisfied in theory by making the right assumptions which we leave as a future work. We measure the values upper bounded in each criterion throughout training of a ResNet-18 on CIFAR10 with and without Dropout. The results are plotted in Figure~\ref{fig:dropout-measurement}. It can be seen that when applying dropout, both Criterion~\ref{crit:norm-perturbation} and~\ref{crit:gradient-alignment} hold in practice. We now provide possible explanations as to why this happens. 
Assuming a positive correlation between gradient's norm and convergence, it is intuitive that replacing an optimized parameter with a random value (such as zero) would increase gradient norm. As a simple example, consider a 1-dimensional quadratic function $(x - a)^2$ where $a$ is a large positive number. After some optimization steps, $x$ will be close to $a$ and replacing it with zero would increase the gradient's norm. While this is not a formal proof, it gives an intuition as to why Criterion~\ref{crit:norm-perturbation} would be satisfied.

However, it is not directly clear why running SGD with dropout  results in the last criterion to be satisfied. In particular, Figure~\ref{fig:dropout-measurement-vs-nodropout} shows that this criterion is not satisfied when running SGD without dropout. We hypothesize that this phenomena is related to the findings of previous work \cite{nichol2018first,dandi_implicit_2021} which shows that SGD imposes an implicit regularization to increase the dot product of consecutive mini-batches. In particular, since the proof in \cite{nichol2018first} allows for a different loss functions per mini-batch, the result is applicable to our case as well and means that SGD with dropout promotes alignment between gradients of random sub-networks. Intuitively, if the gradients of any pair of sub-networks are aligned with each other, it can be expected that they are also aligned with the full network, for example due to existence of sub-networks obtained as a result of pruning that are able to produce similar output as the full network.

\begin{figure}[t]
	\centering
	\includegraphics[width=.8\linewidth]{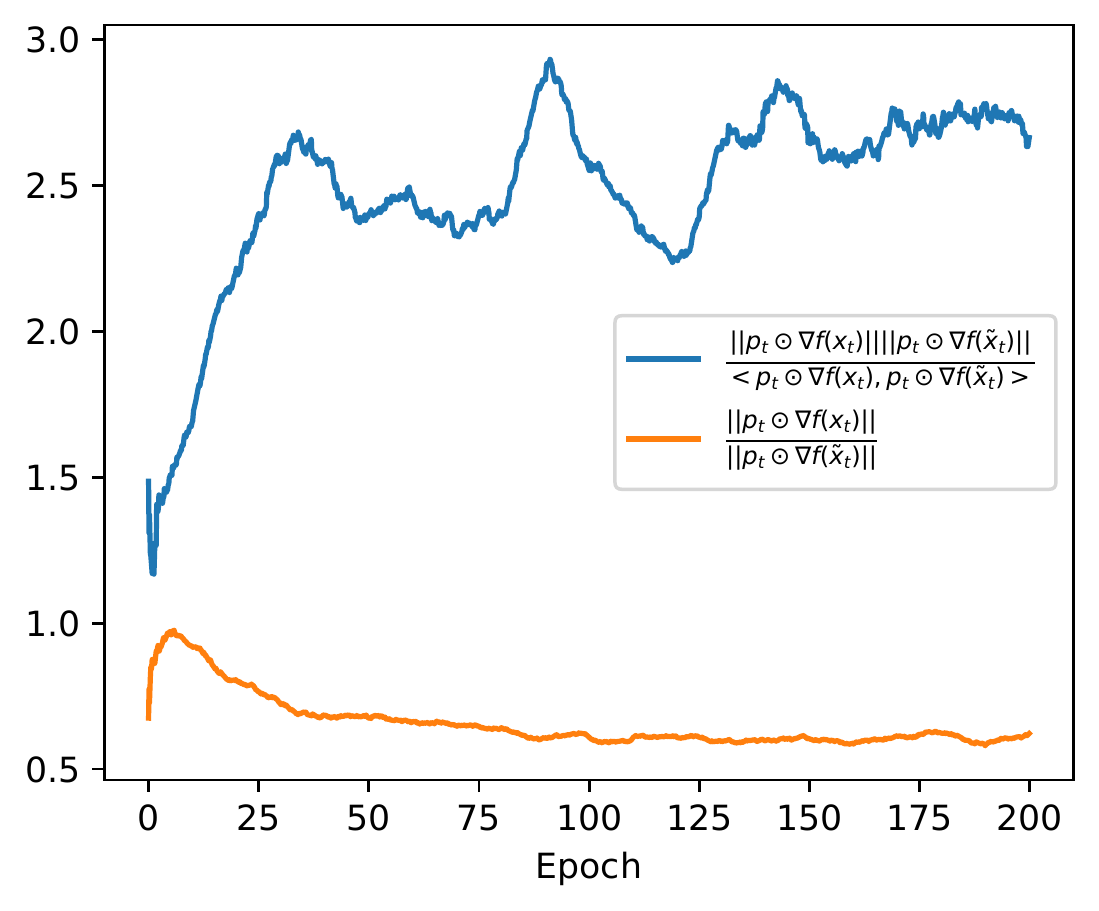}
	\caption{ The values in Criterion~\ref{crit:norm-perturbation} and~\ref{crit:gradient-alignment} measured while training a ResNet-18 on CIFAR10 with  dropout. It can be seen that these criteria are satisfied in practice.}
	\label{fig:dropout-measurement}\vspace{-2em}
\end{figure}

\begin{figure}[t]
	\centering
	\includegraphics[width=.8\linewidth]{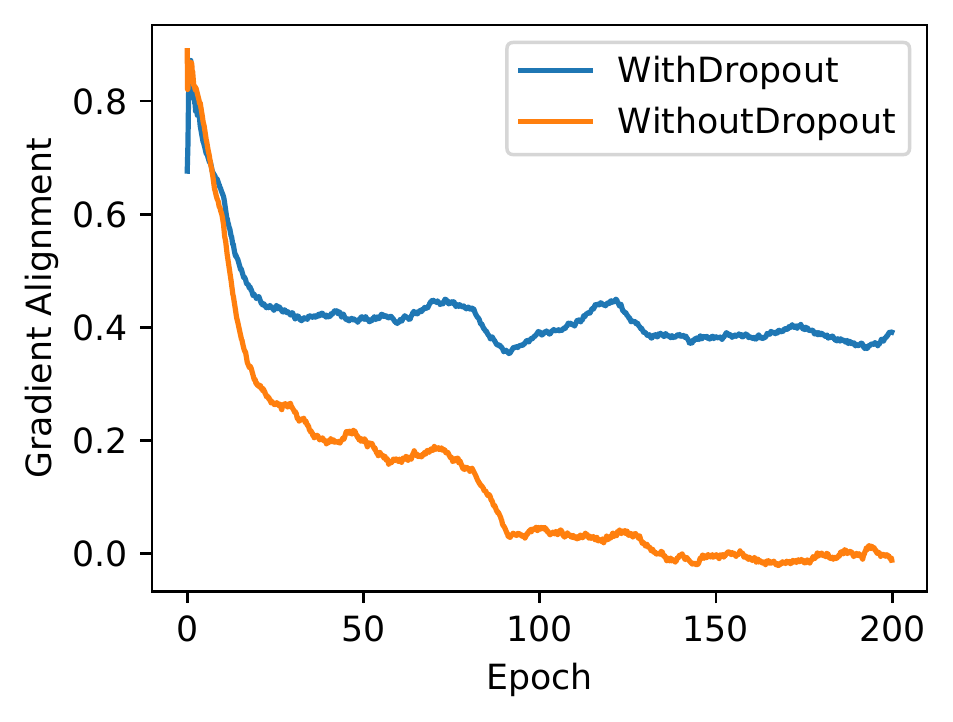}
	\caption{ The value $\frac{\lin{\pp_t \odot \nabla f(\xx_t), \pp_t \odot \nabla f(\tilde{\xx}_t)}}{\norm{\pp_t \odot \nabla f(\xx_t)}\norm{\pp_t \odot\nabla f(\tilde{\xx}_t)}}$, where $\tilde{\xx_t}$ is obtained by applying a random Dropout mask (masking neurons instead of weights), measured while training a ResNet-18 on CIFAR10 with and without dropout. Observably, applying dropout improves gradient alignment.}
	\label{fig:dropout-measurement-vs-nodropout}\vspace{-2em}
\end{figure}

\subsection{Alternating Training Scheme}
\label{sec:thmats}

We now showcase how it is possible to use the criteria we introduced as guidelines for designing new methods. In particular, we design an algorithm for training a given network $\Msuper$, which we refer to as super-network, jointly with one of its sub-networks $\Mcore$, which we refer to as the core network. First, note that any algorithm that trains the full network on some steps and the sub-network on other steps fits the partial SGD template since the parameters of $\Mcore$ can be simply extracted by applying a binary mask~$\1_{\Mcore}$ over~$\xx$, i.e.\ $\1_{\Mcore} \odot \xx$. In order to choose between different switching patterns we compare them in terms of our criteria. In particular, relying on the implicit regularization of SGD that promotes alignment between consecutive mini-batches that we discussed in analyzing dropout, the maximum alignment should be obtained when consecutively switching between training $\Mcore$ and $\Msuper$. This yields Algorithm~\ref{alg:scheme} which we refer to as Alternating Training Scheme (ATS). Note that here we focus on the case where $\Mcore$ is pre-determined. However, if this was not the case, Criterion~\ref{crit:norm-masking} would point toward selecting the sub-network containing largest elements of the gradient vector.

\begin{figure}
	\begin{algorithm}[H]
		\caption{\textsc{Alternating Training %
		}} 
		\label{alg:scheme}
		\resizebox{\linewidth}{!}{
			\begin{minipage}{1.1\linewidth}
				\begin{algorithmic}[1]
					\For{$t=0,\dots,T$}
					\If {$t$ is even}
					\State $\pp_t \gets \1_{\cM_{\text{\small super}}}$ \Comment all-one vector
					\Else
					\State $\pp_t \gets \1_{\cM_{\text{\small core}}}$ \Comment mask of core network
					\EndIf
					\State Take stochastic gradient $\gg_t$ of $\nabla f(\pp_t \odot \xx_t)$
					\State $\xx_{t + 1} \gets \xx_t - \gamma_t \pp_t \odot \gg_t$ 
					\EndFor
				\end{algorithmic}
		\end{minipage}}%
	\end{algorithm}%
	\vspace{-2em}
\end{figure}

ATS is clearly a special case of partial SGD using the same mask sequence $\1_{\Msuper}, \1_{\Mcore}, \1_{\Msuper}, \ldots$ as $\pp_t$ and $\delta \xx_t = -\xx_t$ (thus zeroing weights not part of the core network).  In order to evaluate ATS in terms of our criteria, we follow a similar approach as for Dropout and make the following observations based on training a Transformer jointly with its subnetwork with half the full width:

\paragraph{Large gradient overlap} Let $r$ be the ratio of parameters in the subnetwork to the number of parameters in the full network. For example, in our experiment $r = 0.5$. Given the random initialization the expected value for $\frac{\norm{\nabla f(\xx_t)}^2}{\norm{\pp_t \odot \nabla f(\xx_t}^2}$ is $\frac{1}{r}$ when $t  = 0$.  Figure~\ref{fig:full-masked-grad-ratio} shows that when training using ATS, this quantity remains low and does not increase during training. In fact, it can be seen that when using ATS, the overlap value is significantly improved compared to standard training.%
\paragraph{Norm similarity} Figure~\ref{fig:masked-small-grad-ratio} shows that even when running SGD, the norm of the gradient does not decrease due to applying perturbation. In fact, applying ATS results in less increase in gradient norm as a result of perturbation. However, whether running SGD or ATS the ratio is upper bounded for example by 1. 
\paragraph {Gradient alignment} Finally, we measure $ \frac{\norm{\pp_t \odot \nabla f(\xx_t)}\norm{\pp_t \odot\nabla f(\tilde{\xx}_t)}}{\lin{\pp_t \odot \nabla f(\xx_t), \pp_t \odot \nabla f(\tilde{\xx}_t)}}$ in practice and demonstrate that, as expected, using ATS lowers this value (thus improves alignment) fulfilling Criterion~\ref{crit:gradient-alignment}. The result is plotted Figure~\ref{fig:dot-masked-small}.  %

The convergence bound obtained for ATS from our observations is within a small constant of vanilla SGD. On the other hand, in Section~\ref{sec:experiments}, we observe that the same number of iterations as in standard training is typically sufficient in practice.

\begin{figure*}[t]
	\centering
	\begin{subfigure}[c]{.3\linewidth}
		\includegraphics[width=\linewidth]{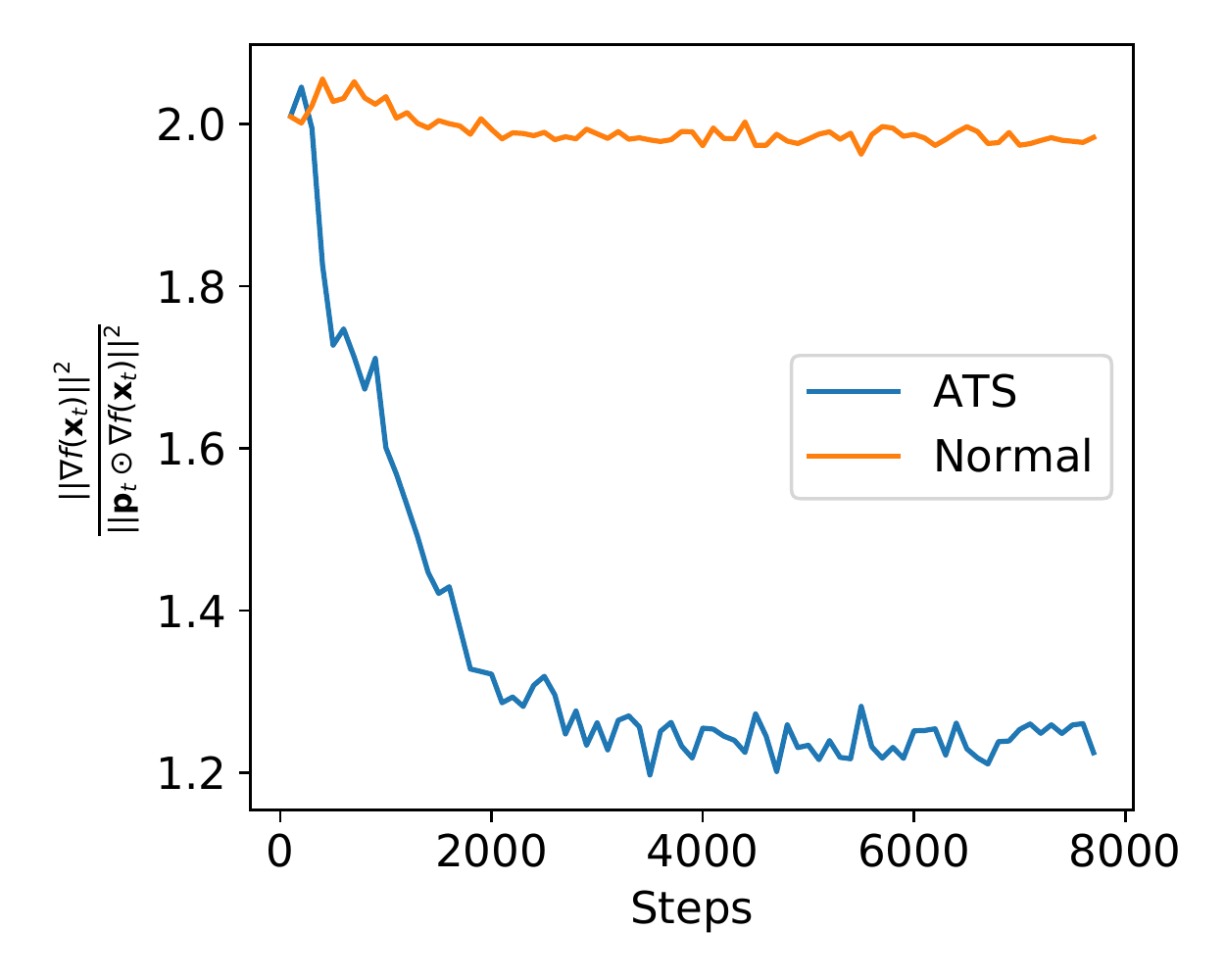}
		\caption{ $\frac{\norm{\nabla f(\xx_t)}^2}{\norm{\pp_t \odot \nabla f(\xx_t}^2}$}
		\label{fig:full-masked-grad-ratio}
	\end{subfigure}
	\begin{subfigure}[c]{.3\linewidth}
		\includegraphics[width=\linewidth]{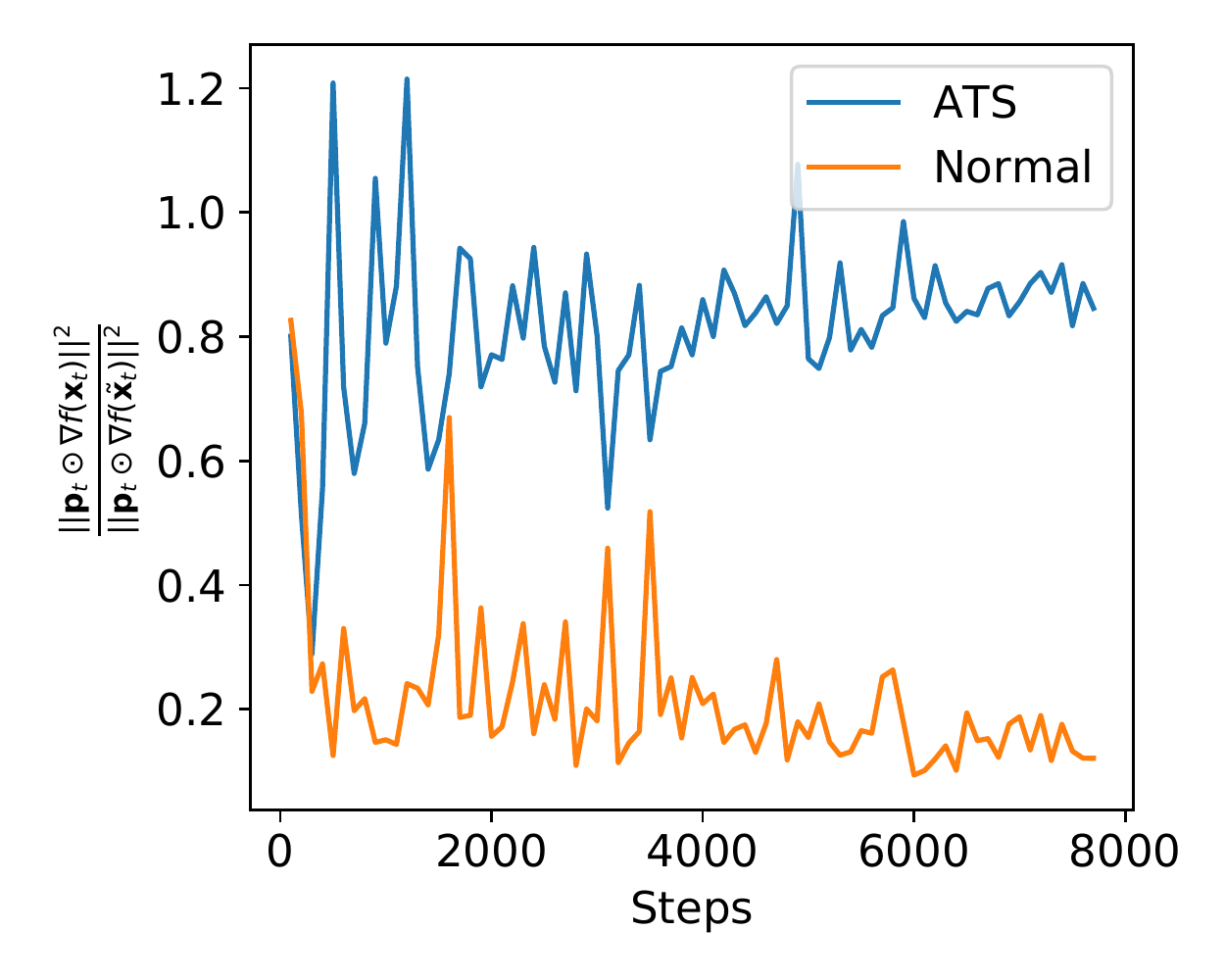}
		\caption{$\frac{\norm{\pp_t \odot \nabla f(\xx_t)}^2}{\norm{\pp_t \odot\nabla f(\tilde{\xx}_t)}^2}$}
		\label{fig:masked-small-grad-ratio}
	\end{subfigure}
	\begin{subfigure}[c]{.3\linewidth}
		\includegraphics[width=\linewidth]{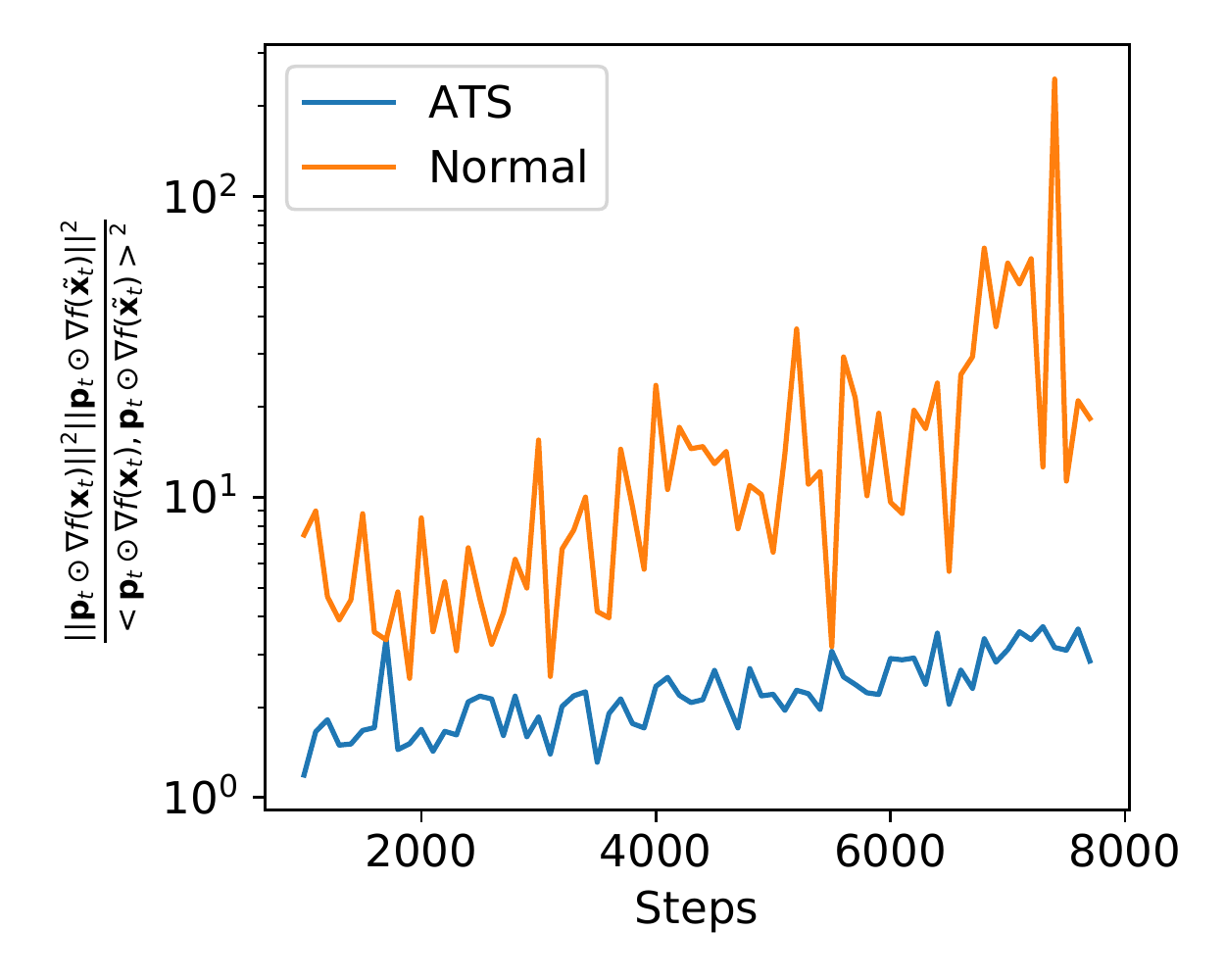}
		\caption{$\frac{\norm{\pp_t \odot \nabla f(\xx_t)}^2\norm{\pp_t \odot\nabla f(\tilde{\xx}_t)}^2}{\lin{\pp_t \odot \nabla f(\xx_t), \pp_t \odot \nabla f(\tilde{\xx}_t)}^2}$}
		\label{fig:dot-masked-small}
	\end{subfigure}
	\caption{Measurement of quantities affecting convergence rate of Theorem~\ref{thm:main} during standard training of a Transformer vs.\ joint training together with its subnetwork with half width using ATS. In all plots, a mini-batch gradient %
	is used to estimate the mentioned quantities, instead of the full gradient.}
\label{fig:ats-measurement}
\end{figure*}

Applying ATS to small-width subnetworks recovers the setting of slimmable nets \cite{yu2018slimmable, yu_universally_2019}. While the original papers apply gradient updates of narrow and wide widths at the same timestep, our scheme unrolls this by updating the small and large parts sequentially. Therefore, the convergence result for ATS also informs the sequential version of slimmable nets, and proves that the $q_t$ 
parameter characterizes the successful joint training. As we discussed earlier,  the sequential version benefits from implicit SGD regularization to keep the gradients of the sub-network and full network aligned.

We also point out that while at first glance the algorithm seems tailored for training a network jointly with one of its sub-networks, the same method can be used to train a pair of arbitrary networks, by using one as the core network, while building the super network $\Msuper$ to be the combination of the two. %
There are various ways to build $\Msuper$. In its simplest form, this can be done by creating a super-network that has a copy of each network and outputs the sum of their outputs. Wide-and-deep networks are an example of using this simple combination. This can be seen by letting $\Mcore$ correspond to the deep network while the wide network acts as the extension network. Using ATS in this setting can allow the wide part to become an extension to the deep part which can focus on solving cases where feature interactions are necessary, which was the intuition behind the design of these networks \cite{cheng2016wide}. Furthermore, the deep part can be used as a standalone network when necessary. 

A more sophisticated combination is summing the output of both models in all hidden layers, instead of just in the output layer. This is the combination scheme we use in our experiments for training a low-rank model. These cases are illustrated in Figure~\ref{fig:super-network-illustration}.

While we have so far focused on training two networks, the scheme can be extended to train several more networks as a hierarchy, to allow more flexibility and fine-grained control over the amount of resource usage. 

\section{Experiments}
\label{sec:experiments}

We consider the application of alternating training scheme in two use cases. In Section~\ref{sec:alternating-lowrank}, we show the effectiveness of ATS by training a standard Transformer together with a Transformer with a low-rank architecture.
 We show that the low-rank model trained with ATS is able to match or in some cases even outperform training the low-rank variant separately. At the same time the full-rank obtains a comparable performance while the total training time is less than training both models separately.  
In Section~\ref{sec:alternating-speed},
	we show that ATS can also be used to train a Transformer together with its small-width sub-network with reduced width and observe similar results to the low-rank training. %
We trained our models using Adam and trained for 22 epochs on WMT17 \cite{bojar2017findings} and for 30 epochs on IWSLT14 \cite{cettolo2014report}. We report a detailed list of hyper-parameters in Appendix~\ref{app:hyperparameters}. 

\subsection{Low-Rank Training}
\label{sec:alternating-lowrank}
In this section we demonstrate how our alternating training scheme can be used to train a low-rank model. In particular, we build a super-model composed of a low-rank as well as a full-rank Transformer by replacing the linear transformations (e.g.\ fully connected layers) in the feed-forward modules and key/query projection in the attention layers of a standard Transformer, with low-rank fully connected (LFC) layer:
\begin{align*}
	LFC(X) := (VU + W) X
\end{align*}
where $U \in \R^{k \times f_{in}}$, $V \in \R^{f_{out} \times k}$, and $W \in \R^{f_{out} \times f_{in}}$ are the parameters. We set $k = r \min(f_{in}, f_{out})$
for all layers where $r$ is the rank reduction ratio. The low-rank variant is the sub-network which only contains~$U$ and $V$ parameters of each LFC layer as well as all other parameters in the super-network that are not in a LFC layer. Note that in this case, the super-network is equivalent to a full-rank variant and can be converted to a standard Transformer after the training. %

We compared the performance of the trained low-rank and full-rank models  with the performance of the same model trained separately using the standard optimizer on IWSLT14 dataset in Table~\ref{tab:transformer-iwslt14-lowrank}.  Observably, the low-rank model trained in our scheme has a superior performance for lower values of $r$ while it matches the performance of standard training for larger $r$. At the same time, the full-rank model has a much higher performance near that of the standard training of the same model, showing the success of ATS in joint training. %
We repeated the same experiment on WMT17 and obtained similar results which we reported in Table~\ref{tab:transformer-wmt-lowrank}.%

\begin{table}[t]
	\caption{BLEU scores of Transformer trained jointly with its \textbf{low-rank variant} using ATS on the IWSLT14 (top) and WMT17 (bottom) machine translation datasets respectively. The sub-network corresponds to the low-rank variant with rank reduced by factor $r$. The baseline (full network with standard training) shows the performance of a standard Transformer. In other rows, the full network corresponds to the super-network which can be converted to a standard Transformer.\vspace{-2mm}}
	\begin{subtable}{\linewidth}
		\centering
		\caption{IWSLT14 dataset.}
		\label{tab:transformer-iwslt14-lowrank}
		\resizebox{\linewidth}{!}{
			\begin{tabular}{lcc}
				\toprule
				Model & Sub-Network  & Full Network\\\midrule
				Standard Training & - & 32.48 (0.32)\\\midrule
				Standard Training ($r = 0.25$)  & 28.96 (0.04) & - \\%
				Alternating Training ($r = 0.25)$ & {29.15 (0.63)} & 31.34 (0.36)\\\midrule
				Standard Training ($r = 0.125$)& 27.87 (0.11) & - \\%
				Alternating Training ($r = 0.125$) & {28.20 (0.48)}  & 31.62 (0.13)	\\\midrule
				Standard Training ($r = 0.03125$)& 17.95 (3.65) & - \\%
				Alternating Training ($r = 0.03125$) & \textbf{24.03 (0.77)} & 31.20 (0.58) \\\bottomrule
			\end{tabular}
		}
	\end{subtable}
	\begin{subtable}{\linewidth}
		\centering
		\vspace{.5em}
		\caption{WMT17 dataset. %
		}
		\label{tab:transformer-wmt-lowrank}
		\resizebox{\linewidth}{!}{
			\begin{tabular}{lcc}
				\toprule
				Model & Sub-Network  & Full Network\\\midrule
				Standard Training & - & 26.41 (0.14) \\\midrule
				Standard Training ($r = 0.03125$)& 23.20 (0.09) & - \\%
				Alternating Training ($r = 0.03125$) & 23.18 (0.12) & 25.50 (0.22)
				\\\bottomrule
			\end{tabular}
		}
		\vspace{-2em}
	\end{subtable}
\end{table}

\subsection{Small-Width Training}
\label{sec:alternating-speed}
In order to show that our model can also be used for training a network together with its sub-network, we train a Transformer with its small-width variant. 
As in the low-rank settings, ATS obtains core networks with similar or even superior performance to solo training as well as super networks with comparable performance. 
We report our result in Appendix~\ref{app:small-width} and also compare per-step computation cost of our method with slimmable training and standard training in FLOPs. It can be seen that ATS is significantly more efficient than slimmable training and even standard training of the super model. %

\section{Future Work}
 Our theoretical framework is widely applicable to a variety of existing and novel training scenarios. Hence, it can be used to analyze other partial training schemes and gain insights useful to improve such schemes. For example, a future direction would be using the analysis done in this work on independent subnet training to allow dynamic switching between local training and global coordination automatically. %
We have also used theoretical insights of our results to design a joint training method for two arbitrary neural networks. While we applied our method to train Transformers on translation tasks, applications in other areas or for other architectures is grounds for future work. Finally, other training schemes can be designed relying on the proposed set of criteria identified in our work.

\section{Conclusion}
In this work, we developed a theoretical framework to analyze convergence of SGD variants that apply parameter perturbation and gradient masking, and showed that our settings cover a large number of scenarios involving partial training. Moreover, we presented a new simultaneous training scheme for arbitrary network pairs, and showed
the efficacy of our scheme for training a Transformer together with its low-rank variant, and alternatively its subnetwork of reduced width.

\clearpage

{\small
		\section*{Acknowledgement}
	This work was supported by Innosuisse Project Privately And by SNSF grant 200020\_200342.
	
	We thank Lie He and Anastasia Koloskova for their help in proofreading of the manuscript.
\bibliographystyle{plainnat}
\bibliography{references}
}

\clearpage
\appendix
\thispagestyle{empty}
\onecolumn \makesupplementtitle

\section{Proof of Theorem~\ref{thm:main}}
\label{app:theorem-proof}

The following proof is inspired by the proof used for BiasedSGD \cite{ajalloeian2020analysis} and adapted for Algorithm~\ref{alg:partialsgd}.

First, let us recall the following remark about the properties of $L$-smooth functions.

\begin{remark}
	$L$-smooth functions satisfy the following \cite[Lemma 1.2.3]{Nesterov2004:book}: 
	\begin{align}
		f(\yy) \leq  f(\xx) + \lin{\nabla f(\xx), \yy - \xx} + \frac{L}{2} \norm{\xx - \yy}^2\,, \forall \xx,\yy \in \R^{d}\,. \label{eq:lsmooth}
	\end{align}
\end{remark}

We now prove the following lemma:

\begin{lemma}
	\label{lemma:main}
	If the assumptions of Theorem~\ref{thm:main} hold:
	\begin{equation}
		\label{eq:smallnc}
		\begin{split}
			\E_{\bxi(\xx_t)} f(\xx_{t + 1}) \leq &   f(\xx_t)  - \frac{\gamma_{base}}{2} \alpha_t^2\norm{\pp_t \odot \nabla f(\tilde{\xx}_t)}^2  + \frac{ \gamma_{base}^2L}{2} \sigma^2 \, ,
		\end{split}
	\end{equation}
\end{lemma}
\begin{proof}
	\begin{align*}
		\E_{\bxi(\xx_t)} f(\xx_{t + 1}) & \leq  f(\xx_t) - \gamma_t \lin{\nabla f(\xx_t), \E [\pp_t \odot \gg_t]} + \frac{\gamma_t^2 L}{2} \E \norm{\pp_t \odot \gg_t}^2\\
		& =  f(\xx_t) - \gamma_t \lin{\nabla f(\xx_t), \pp_t \odot \nabla f(\tilde{\xx}_t)}  + \frac{\gamma_t^2L}{2} (\E \norm{\pp_t \odot (\gg_t - \E \gg_t)}^2  + \norm{\E\gg_t}^2) \\
		& = f(\xx_t) - \gamma_t \lin{\nabla f(\xx_t), \pp_t \odot \nabla f(\tilde{\xx}_t)}  + \frac{\gamma_t^2L}{2} (\E \norm{\pp_t \odot \bxi(\tilde{\xx}_t)}^2 + \norm{\pp_t \odot \nabla f(\tilde{\xx}_t)}^2)\\
		&\leq f(\xx_t)- \gamma_t \lin{\nabla f(\xx_t), \pp_t \odot \nabla f(\tilde{\xx}_t)}  +\frac{\gamma_t^2 L}{2}(M+1)\norm{\pp_t \odot \nabla f(\tilde{\xx}_t)}^2 + \frac{\gamma_t^2L}{2} \sigma^2\\
		&\leq f(\xx_t) - \gamma_t\alpha_t(1 - \frac{\gamma_t}{2\alpha_t} L (M + 1))\norm{\pp_t \odot \nabla f(\tilde{\xx}_t)}^2  + \frac{\gamma_t^2L}{2}  \sigma^2\\
		&\leq f(\xx_t) -  \frac{\gamma_t}{2} \alpha_t\norm{\pp_t \odot \nabla f(\tilde{\xx}_t)}^2  + \frac{\gamma_t^2L}{2}  \sigma^2\\
		&\leq  f(\xx_t) -  \frac{\gamma_{base}}{2} \alpha_t^2\norm{\pp_t \odot \nabla f(\tilde{\xx}_t)}^2  + \frac{\gamma_{base}^2L}{2}  \sigma^2 \,.
	\end{align*}
Where in the last equation we used the facts that $\alpha_t \leq 1$ and $\gamma_t = \alpha_t \gamma_{base}$.	
\end{proof}
We can now prove Theorem~\ref{thm:main}:
\begin{proof}[Proof of Theorem~\ref{thm:main}]
	Recall that we use $\E_{\bxi}$ to refer to the expected over the sequence $\bxi(\xx_0), \ldots, \bxi(\xx_{T-1})$. Define $F_t := \E_{\bxi} f(\xx_t) - f(\xx^\star)$. By rearraging the result of Lemma~\ref{lemma:main} and taking the expectation over all elements in $\bxi$ (in addition to $\bxi(\xx_t)$) we get:
	\begin{align*}
		\frac{1}{2}\E_{\bxi} \alpha_t^2\norm{\pp_t \odot \nabla f(\tilde{\xx}_t)}^2 \leq & \frac{F_t - F_{t + 1}}{\gamma_{base}}  + \frac{ \gamma_{base} L}{2} \sigma^2 \,.
	\end{align*}
	Averaging over all $T$, we get a telescoping summation leading to:
	\begin{align*}
		\frac{1}{T} \sum_{t = 0}^{T - 1} \E_{\bxi} \alpha_t^2\norm{\pp_t \odot \nabla f(\tilde{\xx}_t)}^2 \leq \frac{2F_0}{T\gamma_{base}}  + \gamma_{base} L \sigma^2 \,.
	\end{align*}

	We now use $\gamma_{base}= \min\big\{\frac{1}{L(M+1)}, 
	\frac{\epsilon}{2L\sigma^2}
	\big\}$ to prove the first part of the theorem. Note that this means $\gamma_{base}L\sigma^2 \leq \frac{\epsilon}{2}$. Moreover, $T \geq 4 \left(\frac{\sigma^2}{\epsilon^2}  + \frac{(M + 1)}{\epsilon} \right) \cdot L F_0$ yields $T\gamma_{base} \geq \frac{4F_0}{\epsilon}$ which yields $\frac{2F_0}{T\gamma_{base}} \leq \frac{\epsilon}{2}$. Using these two bounds, proves the first part of the theorem.

	For proving the second part note that:
	\begin{align*}
		\E_{\bxi}\norm{\nabla f(\xx_t)}^2 & =  \frac{\norm{\nabla f(\xx_t)}^2}{\alpha_t^2\norm{\pp_t \odot \nabla f(\tilde{\xx}_t)}^2}  \alpha_t^2\norm{\pp_t \odot \nabla f(\tilde{\xx}_t)}^2\\ 
		& = q_t^2 \alpha_t^2\norm{\pp_t \odot \nabla f(\tilde{\xx}_t)}^2  \\
		& \leq  q^2 \E_{\bxi}\alpha_t^2\norm{\pp_t \odot \nabla f(\tilde{\xx}_t)}^2 \, .
	\end{align*}
	where we used $q \geq q_t$ for all $t \in [T]$ and any sequence of noises $\bxi$, which directly follows from the definition of $q$, to get the last inequality. We can now see that
	\begin{align*}
		  \frac{1}{T}\sum_{t = 0}^{T - 1} \E_{\bxi}\alpha_t^2\norm{\pp_t \odot \nabla f(\tilde{\xx}_t)}^2 \leq \frac{\epsilon}{q^2} \Longrightarrow	 \frac{1}{T}\sum_{t = 0}^{T - 1}\E_{\bxi}\norm{\nabla f(\xx_t)}^2  \leq q^2 \cdot \frac{1}{T}\sum_{t = 0}^{T - 1} \E_{\bxi}\alpha_t^2\norm{\pp_t \odot \nabla f(\tilde{\xx}_t)}^2 \leq \epsilon \,.
	\end{align*}	
	Using $\frac{\epsilon}{q^2}$ as the target threshold in the bound we obtained for the first part of the theorem gives the desired result.%
\end{proof}

\section{Proof  of Lemma~\ref{lemma:bounded-perturbation}}
\label{app:proof-bounded-perturbation-assumption}
\begin{proof}
	We will first prove the bound for $c_{\text{sim}}$. We apply the inequality $\norm{a + b}_2^2 \leq 2\norm{a}_2^2 + 2\norm{b}_2^2$ and get
	\begin{align*}
		\norm{\pp_t \odot \nabla f(\xx_t)}^2  &\leq  2\norm{\pp_t \odot (\nabla f(\xx_t) - \nabla f(\tilde{\xx}_t))}_2^2 + 2\norm{\pp_t \odot \nabla f(\tilde{\xx}_t)}_2^2\\
		& \leq 2\norm{\nabla f(\xx_t) - \nabla f(\tilde{\xx}_t)}_2^2 + 2\norm{\pp_t \odot \nabla f(\tilde{\xx}_t)}_2^2\\
		 & \stackrel{(\ref{eq:lsmoothgrad})}{\leq} 2L^2\norm{\delta\xx_t}^2 +2\norm{\pp_t \odot \nabla f(\tilde{\xx}_t)}^2 \\
		&\stackrel{(\ref{eq:perturbation})}{\leq}  (\frac{1}{2} + 2) \norm{\pp_t \odot \nabla f(\tilde{\xx}_t)}^2 \,.  
	\end{align*}
	Hence $c_{\text{sim}} \leq \frac{\sqrt{10}}{2}$. 
	To prove the bound for alignment note that:
	\begin{align*}
		\lin{\pp_t \odot \nabla f(\xx_t),  \pp_t \odot \nabla f(\tilde{\xx}_t)} &=  \norm{\pp_t \odot \nabla f(\tilde{\xx}_t)}^2 + \lin{\pp_t \odot \nabla f(\tilde{\xx}_t),  \pp_t \odot (\nabla f(\xx_t) -  \nabla f(\tilde{\xx}_t))} \\
		& =  \norm{\pp_t \odot \nabla f(\tilde{\xx}_t)}^2 + \lin{\pp_t \odot \nabla f(\tilde{\xx}_t),  \nabla f(\xx_t) - \nabla f(\tilde{\xx}_t)}\\
		& \geq \norm{\pp_t \odot \nabla f(\tilde{\xx}_t)}^2 - \norm{\pp_t \odot \nabla f(\tilde{\xx}_t)}\norm{\nabla f(\xx_t) - \nabla f(\tilde{\xx}_t)} \\
		&  \stackrel{(\ref{eq:lsmoothgrad})}{\geq} \norm{\pp_t \odot \nabla f(\tilde{\xx}_t)}^2 - L\norm{\pp_t \odot \nabla f(\tilde{\xx}_t)}\norm{\xx_t - \tilde{\xx}_t}\\
		& = (1 - L\frac{\norm{\delta\xx_t}}{\norm{\pp_t \odot \nabla f(\tilde{\xx}_t)}})\norm{\pp_t \odot \nabla f(\tilde{\xx}_t)}^2\\
		&\stackrel{(\ref{eq:perturbation})}{\geq}\frac{1}{2}\norm{\pp_t \odot \nabla f(\tilde{\xx}_t)}^2 \\
		&\geq\frac{1}{2} \cdot \frac{2}{\sqrt{10}}\norm{\pp_t \odot \nabla f(\tilde{\xx}_t)}\norm{\pp_t \odot \nabla f(\xx_t)}
		\, .
	\end{align*}
	Where in the last inequality we used the bound for $c_{\text{sim}}$ which in turn shows $c_{\text{align}} \leq \sqrt{10}$. 
\end{proof}

\section{Small-Width Training Results}
\label{app:small-width}
We obtain the small-width variant by reducing the width of the feed-forward modules and the key/query projections in attention layers by a factor $r$. In particular, we reduce the number of neurons in the hidden layer of the feed-forward modules as well as the key/query embedding dimension of the attention layers  by the factor $r$. In contrast to the low-rank setting where we use an additive combination of the full-rank and low-rank model, we use the standard Transformer as the super-network and train it jointly with its sub-network corresponding to the small-width variant.

\begin{table}[t]
	\caption{BLEU scores of Transformer trained jointly with its \textbf{small width sub-network} on IWSLT14 (left) and WMT17 (right) datasets respectively, with and without ATS. The sub-network width reduction factor is denoted by $r$.\vspace{-2mm}}
	\begin{subtable}{.49\textwidth}
		\centering
		\caption{IWSLT14 dataset. }
		\label{tab:transformer-iwslt14-layer-partition}
		\resizebox{\linewidth}{!}{
			\begin{tabular}{lcc}
				\toprule
				Model & Sub-Network  & Full Network\\\midrule
				Standard Training & - & 32.48 (0.32)\\\midrule
				Standard Training ($r = 0.25$)  & 31.29 (0.27) & - \\%
				Alternating Training ($r = 0.25)$ & {31.59 (0.18)} &31.95 (0.35)\\\midrule
				Standard Training ($r = 0.03125$)& 23.58 (0.58) & - \\%
				Alternating Training ($r = 0.03125$) & 23.84 (0.25) & 30.41(0.40) \\
				\bottomrule
				
			\end{tabular}
		}
	\end{subtable}
	\hfill
	\begin{subtable}{.49\textwidth}
		\centering
		\caption{WMT17 dataset.}
		\label{tab:transformer-wmt-layer-partition}
		\resizebox{\linewidth}{!}{
			\begin{tabular}{lcc}
				\toprule
				Model & Sub-Network  & Full Network\\\midrule
				Standard Training & - & 26.41(0.14) \\\midrule
				Standard Training ($r = 0.03125$)& {20.47 (0.34)} & - \\%
				Alternating Training ($r = 0.03125$) & 20.08 (0.40) & 25.46(0.17) \\\bottomrule
			\end{tabular}
		}
	\end{subtable}
\end{table}

Table~\ref{tab:transformer-iwslt14-layer-partition} includes the performance  of both the full-width and the small-width models trained with this scheme in terms of  the BLEU score on IWSLT14 dataset. Furthermore, the performances of the same models with standard training are also reported. It can be clearly seen that training with this scheme yields better small-width models. Additionally, while the full-width models have a lower performance when trained with ATS than when trained standardly, their performance are still comparable. We also perform the same experiment on WMT17 dataset and report the results in Table~\ref{tab:transformer-wmt-layer-partition} and similarly observe that using ATS, results in a small network matching the performance of standard training and a full model with comparable performance while the cost of training is less than training these two models separately. 

We would like to point out that in comparison to low-rank models, small-width models have a lower performance on the WMT17 dataset for small values of $r$ while the number of parameters remain comparable. In particular, standard training of a low-rank Transformer with $r = 0.015625$ (which has the same number of parameter as low-width Transformer with $r = 0.03125$) yields 22.10 BLEU score (vs 20.47 BLEU score). This encourages using low-rank models instead of low-width models.

For completeness, given the similarity of our method and slimmable training we compare the two by performing the small-width experiments with $r = 0.5$ on a ResNet-18 over CIFAR10 datasets. The results are reported in Table~\ref{tab:resnet18-cifar10-ats-vs-slimmable}. We additionally reported an estimate of the number of FLOPs per epoch performed by each method. It can be seen that while the performance of both methods are similar, ATS is much more efficient.

\begin{table}[t]
	\centering
	\caption{Final model accuracy and per-epoch FLOPs count of ATS and slimmable training for ResNet-18 over CIFAR10 dataset.}
	\label{tab:resnet18-cifar10-ats-vs-slimmable}
	\begin{tabular}{lccc}
		\toprule
		Method & Model & Accuracy  & $10^{12}$ FLOPs\\\midrule
		\multirow{2}{*}{ATS} & Core & 93.94 & \multirow{2}{*}{31.3}  \\
		& Full  & 95.05 &  \\\midrule%
		\multirow{2}{*}{Slimmable} & Core  & 93.99 & \multirow{2}{*}{62.7} \\
		 & Full & 94.72 & \\\bottomrule
	\end{tabular}
	\end{table}

\section{Hyperparameters for Experiments}
\label{app:hyperparameters}
We train each experiment with IWSLT14 dataset on a single GTX Titan X  GPU and use a Tesla V100 GPU for training WMT17 experiments. For the implementation of our experiments we extend the implementation of transformer in FairSeq \cite{ott2019fairseq} and use its CLI (\verb!fairseq-train!) for training. The hyperparameters based on the dataset are reported in Table~\ref{tab:transformer-hyperparams}. 
\begin{table}
		\caption{Hyper-parameters used for training on different datasets}
		\label{tab:transformer-hyperparams}
		\centering
	\begin{tabular}{l|c|c}
		\toprule
		Parameter&\hspace{1cm}WMT17\hspace{1cm}~&IWSLT14\hfill\\\midrule
		adam\_betas&\multicolumn{2}{c}{(0.9, 0.98)}\\\midrule
		adam\_eps&\multicolumn{2}{c}{1e-09}\\\midrule
		alternate\_lr\_coef&\multicolumn{2}{c}{1.0}\\\midrule
		criterion&\multicolumn{2}{c}{label\_smoothed\_cross\_entropy}\\\midrule
		decoder\_attention\_heads&8&4\\\midrule
		decoder\_embed\_dim&\multicolumn{2}{c}{512}\\\midrule
		decoder\_ffn\_embed\_dim&2048&1024\\\midrule
		decoder\_layers&\multicolumn{2}{c}{6}\\\midrule
		dropout&0.1&0.3\\\midrule
		encoder\_attention\_heads&8&4\\\midrule
		encoder\_embed\_dim&\multicolumn{2}{c}{512}\\\midrule
		encoder\_ffn\_embed\_dim&2048&1024\\\midrule
		encoder\_layers&\multicolumn{2}{c}{6}\\\midrule
		label\_smoothing&\multicolumn{2}{c}{0.1}\\\midrule
		lr&0.0007&0.0005\\\midrule
		lr\_scheduler&\multicolumn{2}{c}{inverse\_sqrt}\\\midrule
		max\_epoch&22&30\\\midrule
		max\_tokens&18000&12000\\\midrule
		optimizer&\multicolumn{2}{c}{adam\_for\_partitioned\_model}\\\midrule
		partitioned\_linear\_module\_name&\multicolumn{2}{c}{PartitionedLinear or PartitionedRankLinear}\\\midrule
		share\_decoder\_input\_output\_embed&False&True\\\midrule
		small\_ratio&\multicolumn{2}{c}{0.03125}\\\midrule
		task&\multicolumn{2}{c}{translation\_with\_partitioned\_model}\\\midrule
		warmup\_init\_lr&1e-07&-1\\\midrule
		warmup\_updates&\multicolumn{2}{c}{4000}\\\midrule
		weight\_decay&0.0&0.0001\\\bottomrule
		\end{tabular}
	\end{table}
For training the sub-network alone we additionally set the parameter "max\_level" to 1. For training with alternating scheme we set the parameter ``alternate" to True. Training the normal Transformer (with or without ATS) for one epochs take approximately 7 minutes on IWSLT14 dataset and 100 minutes on WMT17 dataset. 

For CIFAR10, we train using SGD with momentum $0.9$ and weight decay $5 \times 10^{-4}$. Separate batch normalization layers were used for the core and the super network. For measuring alignment with and without dropout (Figure~\ref{fig:dropout-measurement}), batch normalization was disabled.

\end{document}